\documentclass{article}
\usepackage{amsthm}
\usepackage{amssymb}
\usepackage[utf8]{inputenc}
\usepackage{amsmath}
\usepackage{fullpage}
\usepackage{graphicx}
\usepackage{hyperref}
\usepackage{color}
\usepackage[boxruled]{algorithm2e}
\newtheorem{definition}{Definition}

\newtheorem{fact}{Fact}

\newtheorem{theorem}{Theorem}
\newtheorem{lemma}{Lemma}

\newcommand{\impr}{I^*}
\newcommand{\reg}{\mbox{regret}}
\newcommand{\Ex}{\mathbb{E}}
\newcommand{\ExR}[1]{\mathbb{E}\left[ #1 \right]}
\newcommand{\PrR}[1]{\Pr\left( #1 \right)}
\newcommand{\condR}{\ \vline\  }

\newcommand{\comment}[1]{}

\newcommand{\xtwo}{x_i}
\newcommand{\xone}{y_i}
\newcommand{\Emu}{E^{\mu}}
\newcommand{\Etheta}{E^{\theta}}
\newcommand{\Beta}{\text{Beta}}

\title{Further Optimal Regret Bounds for Thompson Sampling}
 \author{Shipra Agrawal\\
 {shipra@microsoft.com}\\
 Microsoft Research India 
 \and
 {Navin Goyal} \\
 {navingo@microsoft.com}\\
 Microsoft Research India
 }
\begin{document}

\maketitle
\begin{abstract}
Thompson Sampling is one of the oldest heuristics for multi-armed bandit problems. It is a randomized algorithm based on Bayesian ideas, and has recently generated significant interest after several studies demonstrated it to have better empirical performance compared to the state of the art methods. In this paper, we provide a novel regret analysis for Thompson Sampling that simultaneously proves both the optimal problem-dependent bound of $(1+\epsilon)\sum_i \frac{\ln T}{\Delta_i}+O(\frac{N}{\epsilon^2})$ and the first near-optimal problem-independent bound of $O(\sqrt{NT\ln T})$ on the expected regret of this algorithm. Our near-optimal problem-independent bound solves a COLT 2012 open problem of
Chapelle and Li.  The optimal problem-dependent regret bound for this problem was first
proven recently by Kaufmann~et~al.~\cite{KaufmannMunos12}. 
Our novel martingale-based analysis techniques are conceptually simple, easily extend to distributions other than the Beta distribution, and also extend to the more general contextual bandits setting \cite{agrawal-contextual}. 
\end{abstract}
\section{Introduction}
Multi-armed bandit problem models the exploration/exploitation trade-off inherent in sequential decision problems. 
One of the early motivations for studying MAB problem was clinical trials: suppose that we have $N$ different treatments of unknown efficacy for a certain disease. Patients arrive sequentially, and we must decide on a treatment to administer for each arriving patient. To make this decision, we could learn from how the previous choices of treatments fared for the previous patients. After a sufficient number of trials, we may have a reasonable idea of which treatment is most effective, and from then on, we could administer that treatment for all the patients. However, initially, when there is no or very little information available, we need to \emph{explore} and try each treatment sufficient number of times. We wish to do this exploration in such a way that we can find the best treatment and start \emph{exploiting} it as soon as possible. The MAB problem is to decide how to choose the treatment for the next patient, given the outcomes of the treatments so far. Today, multi-armed bandit problem has a diverse set of applications some of which will be mentioned shortly. 

 Many versions and generalizations of the multi-armed bandit problem have been studied in the literature; in this paper we will consider a basic
and well-studied version of this problem: the stochastic multi-armed bandit problem. 
 Among many algorithms available for the stochastic bandit problem, some popular ones include Upper Confidence Bound (UCB) family of algorithms, (e.g., \cite{LaiR, Aueretal}, and more 
recently \cite{AudibertB09, Garivier11, Maillard11, Kaufmann12}), which have good theoretical guarantees, and the algorithm by \cite{Gittins}, which gives optimal strategy under Bayesian setting with known priors and geometric time-discounted rewards. 
 In one of the earliest works on stochastic bandit problems, \cite{Thompson} proposed a natural randomized Bayesian algorithm to minimize regret. 
  The basic idea is to assume a simple prior distribution on the parameters of the reward distribution of every arm, and at any time step, play an arm according to its posterior probability of being the best arm. This algorithm is known as \emph{Thompson Sampling} (TS), and it is a member of the family of \emph{randomized probability matching} algorithms. TS is a very natural
algorithm and the same idea has been rediscovered many times independently 
in the context of reinforcement learning, e.g., in \cite{Wyatt97, OrtegaB10, DBLP:conf/icml/Strens00}.
 We emphasize that although TS algorithm is a Bayesian approach, the description of the algorithm and our analysis apply to the prior-free stochastic multi-armed bandit model where parameters of the reward distribution of every arm are fixed, though unknown (see Section \ref{sec:MAB}). One could interpret the ``assumed" Bayesian priors as the current knowledge of the algorithm about the arms. 
Thus, our regret bounds for Thompson Sampling are directly comparable to the regret bounds for UCB family of algorithms which are a frequentist approach to the same problem. 
 
Recently, TS has attracted considerable attention. Several studies (e.g., \cite{Granmo, Scott, DBLP:conf/nips/ChapelleL11, MayL}) have empirically demonstrated the efficacy of Thompson Sampling: \cite{Scott} provides a detailed discussion of probability matching techniques in many general settings along with favorable empirical comparisons with other techniques. \cite{DBLP:conf/nips/ChapelleL11} demonstrate that empirically TS achieves regret comparable to the lower bound of \cite{LaiR}; and in applications like display advertising and news article recommendation, it is competitive to or better than popular methods such as UCB. In their experiments, TS is also more robust to delayed or batched feedback (delayed feedback means that the result of a play of an arm may become available only after some time delay, but we are required to make immediate decisions for which arm to play next) than the other methods. 
A possible explanation may be that TS is a randomized algorithm and so it is unlikely to 
get trapped in an early bad decision during the delay. 
Microsoft's adPredictor (\cite{GraepelCBH10}) for CTR prediction of search ads on Bing uses the idea of Thompson Sampling. 

Despite being easy to implement, competitive to the state of the art methods, and popular in practice,
TS lacked a strong theoretical analysis. \cite{Granmo, MayKLL} provide weak guarantees, namely, a bound of $o(T)$ on expected regret in time $T$. Significant progress was made in the recent work of \cite{AgrawalG12} and \cite{KaufmannMunos12}.  In \cite{AgrawalG12}, the first logarithmic bound on expected regret of TS algorithm were proven. \cite{KaufmannMunos12} provided a bound that matches the asymptotic lower bound of~\cite{LaiR} for this problem. However, both these bounds were problem dependent, i.e. the regret bounds are logarithmic in $T$ when the problem parameters, namely the mean rewards for each arm, and their differences, are assumed to be constants. The problem-independent bounds implied by these existing works were far from optimal. Obtaining a problem-independent bound that is close to the lower bound of $\Omega(\sqrt{NT})$ was also posed as an open problem by Chapaelle and Li~\cite{ChapelleL12}.

In this paper, we give a regret analysis for Thompson Sampling that provides both optimal problem-dependent and near-optimal problem-independent regret bounds for Thompson Sampling. Our novel martingale-based analysis technique is conceptually simple
(arguably simpler than the previous work). Our technique easily extends to distributions other than Beta distribution, and it also extends to the more general contextual bandits setting \cite{agrawal-contextual}. While the basic idea for the analysis in the contextual bandits setting of \cite{agrawal-contextual} is inspired by the idea in this paper, the details are substantially different.

Before stating our results, we describe the MAB problem and the TS algorithm formally. 



\subsection{The multi-armed bandit problem} 
\label{sec:MAB}
We consider the stochastic multi-armed bandit (MAB) problem: We are given a slot machine with $N$ arms; 
at each time step $t=1, 2, 3, \ldots $, one of the $N$ arms must be chosen to be played. 
Each arm $i$, when played, yields a random real-valued reward according to some fixed (unknown) distribution with support in $[0,1]$. 
The random reward obtained from playing an arm repeatedly are i.i.d. and independent of the plays of the other arms. 
The reward is observed immediately after playing the arm. 

An algorithm for the MAB problem must decide which arm to play at each time step $t$, based on the outcomes of the previous $t-1$ plays. 
Let $\mu_i$ denote the (unknown) expected reward for arm $i$.
A popular goal is to maximize the expected total reward in time $T$, i.e., $\Ex[\sum_{t=1}^{T} \mu_{i(t)}]$, where $i(t)$ is the arm played 
in step $t$, and the expectation is over the random choices of $i(t)$ made by the algorithm. 
It is more convenient
to work with the equivalent measure of expected total \emph{regret}: the amount we lose because of not playing optimal arm in each step. 
To formally define regret, let us introduce some notation. 
Let $\mu^*:=\max_i \mu_i$, and $\Delta_i := \mu^* - \mu_i$. Also, let $k_i(t)$ denote the number of times arm $i$ has been played up to step $t-1$.  Then the
expected total regret in time $T$ is given by \vspace{-0.05in}
$$\mbox{$\ExR{{\cal R}(T)}= \ExR{\sum_{t=1}^{T} (\mu^{*}-\mu_{i(t)})} = \sum_{i} \Delta_i \cdot \ExR{k_i(T+1)}. $}$$
Other performance measures include PAC-style guarantees; we do not consider those measures here. 


\subsection{Thompson Sampling}
We provide the details of Thompson Sampling algorithm and our analysis for the Bernoulli bandit problem, i.e. when the rewards are either $0$ or $1$, and for arm $i$ the probability of success (reward =$1$) is $\mu_i$. This description of Thompson Sampling follows closely that of \cite{DBLP:conf/nips/ChapelleL11}. A simple extension of this algorithm to general reward distributions with support $[0,1]$ is described in \cite{AgrawalG12}, which seamlessly extends our analysis for Bernoulli bandits to general stochastic bandit problem.

The algorithm for Bernoulli bandits maintains Bayesian priors on the Bernoulli means $\mu_i$'s. Beta distribution turns out to be a very convenient 
choice of priors for Bernoulli rewards. 
Let us briefly recall that beta distributions form a family of continuous probability distributions on the interval $(0,1)$. The pdf of $\Beta(\alpha, \beta)$, the beta distribution with parameters $\alpha > 0$, $\beta > 0$, is given by $f(x; \alpha, \beta) = \frac{\Gamma(\alpha+\beta)}{\Gamma(\alpha)\Gamma(\beta)}x^{\alpha-1}(1-x)^{\beta-1}$. The mean of $\Beta(\alpha, \beta)$ is $\alpha/(\alpha+\beta)$; and as is apparent from the pdf, higher the $\alpha, \beta$, tighter is the concentration of $\Beta(\alpha, \beta)$ around the mean. 
Beta distribution is useful for Bernoulli rewards 
because if the prior is a $\Beta(\alpha, \beta)$ distribution, then after observing a Bernoulli trial, the posterior distribution is simply $\Beta(\alpha+1, \beta)$ or $\Beta(\alpha, \beta+1)$, depending on whether the trial resulted in a success or failure, respectively. 

The Thompson Sampling algorithm initially assumes arm $i$ to have prior $\Beta(1, 1)$ on $\mu_i$, which is
natural because $\Beta(1,1)$ is the uniform distribution on $(0,1)$. At time $t$, having observed $S_i(t)$ successes (reward = $1$) and $F_i(t)$  failures (reward = $0$) in $k_i(t) = S_i(t)+F_i(t)$ plays of arm $i$, the algorithm updates the distribution on $\mu_i$ as $\Beta(S_i(t)+1, F_i(t)+1)$. 
The algorithm then samples from these posterior distributions of the $\mu_i$'s, and plays an arm according to the probability of its mean being the largest. We summarize the Thompson Sampling algorithm below. 

  \begin{algorithm}[H]
  \caption{Thompson Sampling for Bernoulli bandits}
For each arm $i=1,\ldots, N$ set $S_i=0, F_i=0$.\\
\ForEach{$t=1, 2, \ldots,$}{
		For each arm $i=1,\ldots, N$, sample $\theta_i(t)$ from the $\Beta(S_i+1, F_i+1)$ distribution. \\
  	Play arm $i(t) := \arg \max_i \theta_i(t)$ and observe reward $r_t$.\\
  	If $r_t=1$, then $S_{i(t)}=S_{i(t)}+1$, else $F_{i(t)}=F_{i(t)}+1$.
  	}
\end{algorithm} \vspace{-0.05in}

\subsection{Our results}
In this article, we bound the \emph{finite time} expected regret of Thompson Sampling. 
From now on we will assume that the first arm is the unique optimal arm, i.e., $\mu^* = \mu_1 > \arg \max_{i\ne 1} \mu_i$. Assuming that the first arm is an optimal arm is a matter of convenience for stating the results and for the analysis and of course
the algorithm does not use this assumption. 
The assumption of \emph{unique} optimal arm is also without loss of generality, since adding more arms with $\mu_i=\mu^*$ can only decrease the expected regret; details of this argument were provided in \cite{AgrawalG12}.
\begin{theorem}
\label{th:problem-dependent} {\bf(Problem-dependent bound)}
For the $N$-armed stochastic bandit problem, Thompson Sampling algorithm has expected regret
$$\Ex[{\cal R}(T)] \leq (1+\epsilon) \sum_{i=2}^N \frac{\ln T}{d(\mu_i,\mu_1)} \Delta_i+ O(\frac{N}{\epsilon^2})  $$
in time $T$, where $d(\mu_i,\mu_1)=\mu_i \log \frac{\mu_i}{\mu_1} + (1-\mu_i) \log \frac{(1-\mu_i)}{(1-\mu_1)}$. The big-Oh notation \footnote{For any two functions $f(n),g(n)$, $f(n)=O(g(n))$ if there exist two constants $n_0$ and $c$ such that for all $n\ge n_0$, $f(n) \le cg(n)$.} in above assumes $\mu_i, \Delta_i, i=1, \ldots, N$ to be constants.
\end{theorem}
\begin{theorem}
\label{th:problem-independent} {\bf(Problem-independent bound)}
For the $N$-armed stochastic bandit problem, Thompson Sampling algorithm has expected regret
$$\Ex[{\cal R}(T)] \leq O(\sqrt{NT\ln T})$$
in time $T$, where the big-Oh notation hides only the absolute constants.
\end{theorem}
Let us contrast our bounds with the previous work. Let us first consider the problem-dependent regret bounds, i.e., regret bounds that depend on problem parameters $\mu_i, \Delta_i, i=1,\ldots, N$. Lai and Robbins~\cite{LaiR} essentially proved the following lower bound on the regret of any bandit algorithm (see \cite{LaiR} for a precise statement): 
$$\Ex[{\cal R}(T)] \geq \left[\sum_{i=2}^{N} \frac{\Delta_i}{d(\mu_i, \mu_1)} + o(1) \right] \ln{T}. $$
They also gave algorithms asymptotically achieving this guarantee, though unfortunately their algorithms are not efficient. Auer et al.~\cite{Aueretal} gave the UCB1 algorithm, which is efficient and achieves the following bound: 
$$\Ex[{\cal R}(T)] \le \left[8 \sum_{i=2}^N \frac{1}{\Delta_i} \right] \ln{T} + (1+\pi^2/3)\left(\sum_{i=2}^{N} \Delta_i\right).$$
More recently, Kaufmann et al.~\cite{Kaufmann12} gave Bayes-UCB algorithm which achieves 
the lower bound of \cite{LaiR} for Bernoulli rewards. Bayes-UCB is a UCB like algorithm, where the upper confidence bounds are based on the quantiles of Beta posterior distributions. Interestingly, these upper confidence bounds turn out to be similar to those used by algorithms in \cite{Garivier11} and \cite{Maillard11}. Our bounds in Theorem \ref{th:problem-dependent} achieve the asymptotic lower bounds of \cite{LaiR}, and match those provided by \cite{KaufmannMunos12} for Thompson Sampling.

 Theorem \ref{th:problem-independent} shows that Thompson Sampling also achieves a problem independent regret bound of $O(\sqrt{NT\ln T})$ on regret. This is the first analyis for TS that matches the $\Omega(\sqrt{NT})$ problem-inpdependent lower bound (see \cite{bubeck:regret}) for this problem within logarithmic factors. 
The  problem-dependent bounds in the existing work implied only suboptimal problem-independent bounds: \cite{AgrawalG12} implied a problem independent bound of $O(T^{2/3})$. In \cite{KaufmannMunos12}, the additive problem dependent term was not explicitly calculated, which makes it difficult to derive the corresponding problem independent bound, but on a preliminary examination, it appears that it would involve an even higher power of $T$. To compare with other existing algorithms for this problem, note that the best known problem-independent bound for the expected regret of UCB1 is also $O(\sqrt{NT\ln T})$ (see \cite{bubeck:regret}). More recently, Audibert and Bubeck~\cite{AudibertB09} gave an algorithm MOSS, inspired by UCB1, with regret $O(\sqrt{NT})$.

\section{Proofs}
In this section, we prove Theorem \ref{th:problem-dependent} and Theorem \ref{th:problem-independent}. The proofs of the two theorems follow the same steps, and diverge only towards the end of the analysis.

\paragraph{Proof Outline:} Our proof uses a martingale based analysis. Essentially, we prove that conditioned on any history of execution in the preceding steps, the probability of playing any suboptimal arm $i$ at the current step can be bounded by a linear function of the probability of playing the optimal arm at the current step. This is proven in Lemma \ref{lem:main}, which forms the core of our analysis. Further, we show that the coefficient in this linear function decreases exponentially fast with the increase in the number of plays of optimal arm (refer to Lemma \ref{lem:pBound}), this allows us to bound the total number of plays of every suboptimal arm, to bound the regret as desired. The difference between the analysis for obtaining the logarithmic problem-dependent bound of Theorem \ref{th:problem-dependent}, and the problem-independent bound of Theorem \ref{th:problem-independent} is merely technical, and occurs only towards the end of the proof.\newline\\
We recall some of the definitions introduced earlier, and introduce some new notations used in the proof.
$F^B_{n,p}(\cdot)$ denotes the cdf and $f^B_{n,p}(\cdot)$ denotes the probability mass function of the binomial distribution 
with parameters $n, p$. Let $F^{beta}_{\alpha, \beta}(\cdot)$ denote the cdf of the beta distribution with parameters $\alpha, \beta$.

\begin{definition}
$k_i(t)$ is defined as the number of plays of arm $i$ until time $t-1$, and $S_i(t)$ as the number of successes among the plays of arm $i$ until time $t-1$. Also, $i(t)$ denotes the arm played at time $t$.
\end{definition}
\begin{definition}
For each arm $i$, we will choose two thresholds $\xtwo$ and $\xone$ such that $\mu_i<\xtwo<\xone < \mu_1$. 
The specific choice of these thresholds will depend on whether we are proving problem-dependent bound or problem-independent 
bound, and will be described at the approporiate points in the proof. 
Define $L_i(T)=\frac{\ln T}{d(\xtwo, \xone)}$, and $\hat{\mu}_i(t) = S_i(t)/k_i(t)$ (define $\hat{\mu}_i(t)=1$ when $k_i(t)=0$). Define $\Emu_i(t)$ as the event that $\hat{\mu}_i(t) \le \xtwo$. Define $\Etheta_i(t)$ as the event that $\theta_i(t) \le \xone$. 

Intuitively, $\Emu_i(t)$, $\Etheta_i(t)$ are the events that $\hat{\mu}_i(t)$ and $\theta_i(t)$, respectively, are not too far from the mean $\mu_i$. As we show later, these events will hold with high probability for most time steps.
\end{definition}
\begin{definition}
Define filtration ${\cal F}_{t-1}$ as the history of plays until time $t-1$, i.e. 
$${\cal F}_{t-1} = \{i(w), r_{i(w)}(w), i=1,\ldots, N, w=1,\ldots, t-1\},$$
where $i(t)$ denotes the arm played at time $t$, and $r_{i}(t)$ denotes the reward observed for arm $i$ at time $t$.
\end{definition}
\begin{definition}
Define, $p_{i,t}$ as the probability 
$$ p_{i,t}= \Pr(\theta_1(t) > \xone | {\cal F}_{t-1}).$$
Note that $p_{i,t}$ is determined by ${\cal F}_{t-1}$. 
\end{definition}

\begin{lemma}
\label{lem:main}
For all $t \in [1,T]$, and $i\ne 1$,
$$ \PrR{i(t)=i, \Emu_i(t), \Etheta_i(t)) \condR {\cal F}_{t-1}} \le  \frac{(1-p_{i,t})}{p_{i,t}} \PrR{i(t)=1, \Emu_i(t), \Etheta_i(t) \condR {\cal F}_{t-1}},$$
where $p_{i,t}= \Pr(\theta_1(t) > \xone | {\cal F}_{t-1})$.
%
\end{lemma}

\begin{proof}
Note that whether $\Emu_i(t)$ is true or not is determined by ${\cal F}_{t-1}$. Assume that filtration ${\cal F}_{t-1}$ is such that $\Emu_i(t)$ is true (otherwise the probability on the left hand side is $0$ and the inequality is trivially true).
It then suffices to prove that  
\begin{equation} \label{eqn:suff}
\PrR{i(t)=i  \condR \Etheta_i(t), {\cal F}_{t-1}} \le  \frac{(1-p_{i,t})}{p_{i,t}} \PrR{i(t)=1 \condR \Etheta_i(t), {\cal F}_{t-1}}. 
\end{equation}
Let $M_i(t)$ denote the event that arm $i$ exceeds all the suboptimal arms at time $t$. That is, 
$$M_i(t): \theta_{i}(t) \ge \theta_j(t), \forall j\ne 1.$$

We will prove the following two inequalities which immediately give \eqref{eqn:suff}.
\begin{eqnarray}
\PrR{i(t)=1 \condR \Etheta_i(t), {\cal F}_{t-1}} &\ge& p_{i,t} \cdot \PrR{M_i(t) \condR  \Etheta_i(t), {\cal F}_{t-1}}, \label{eqn:it1} \\
\PrR{i(t) = i \condR   \Etheta_i(t),{\cal F}_{t-1}} &\le&  (1-p_{i,t}) \cdot \PrR{M_i(t) \condR  \Etheta_i(t), {\cal F}_{t-1}}. \label{eqn:iti} 
\end{eqnarray}

We have 
\begin{equation} \label{eqn:M}
\PrR{i(t)=1 \condR \Etheta_i(t), {\cal F}_{t-1}} \ge \PrR{i(t)=1, M_i(t) \condR \Etheta_i(t), {\cal F}_{t-1}} = \PrR{M_i(t) \condR \Etheta_i(t), {\cal F}_{t-1}} \cdot \PrR{i(t)=1 \condR M_i(t), \Etheta_i(t), {\cal F}_{t-1}}.
\end{equation}
Now, given $M_i(t), \Etheta_i(t)$, it holds that
for all $j\ne i, j\ne 1$,
$$\theta_j(t) \le \theta_i(t) \le \xone,$$
and so
$$\Pr(i(t)=1 \condR M_i(t), \Etheta_i(t), {\cal F}_{t-1}) \ge \Pr(\theta_1(t) > \xone \condR M_i(t), \Etheta_i(t), {\cal F}_{t-1}) = \Pr(\theta_1(t) > \xone \condR {\cal F}_{t-1}) = p_{i,t} .$$
The second last equality follows because the events $M_i(t)$ and $\Etheta_i(t), \forall i\ne 1$ involve conditions on only $\theta_j(t)$ for $j\ne 1$, and given ${\cal F}_{t-1}$ (and hence $\hat{\mu}_j(t), k_j(t), \forall j$), $\theta_1(t)$ is independent of all the other $\theta_j(t), j\ne 1$, and hence independent of these events.
This together with \eqref{eqn:M} gives \eqref{eqn:it1}.

Since $\Etheta_i(t)$ is the event that $\theta_i(t) \le \xone$, therefore, given $\Etheta_i(t)$, $i(t)=i$ only if $\theta_1(t) <\xone$.
This gives \eqref{eqn:iti}:
\begin{eqnarray*}
\PrR{i(t)=i \condR   \Etheta_i(t),{\cal F}_{t-1}} & \le & \PrR{\theta_1(t)\le \xone, \theta_i(t) \ge \theta_j(t), \forall j\ne 1, \condR \Etheta_i(t),  {\cal F}_{t-1}} \\
& = & \PrR{\theta_1(t) \le \xone \condR  {\cal F}_{t-1}}\cdot \PrR{\theta_i(t) \ge \theta_j(t), \forall j\ne 1 \condR  \Etheta_i(t), {\cal F}_{t-1}}\\ 
& = & (1-p_{i,t}) \cdot \Pr(M_i(t) \condR  \Etheta_i(t), {\cal F}_{t-1}).
\end{eqnarray*}


\end{proof}

\begin{lemma}
\label{lem:E(t)}
$$\sum_{t=1}^T \PrR{i(t) = i, \overline{\Emu_i(t)}} \le \frac{1}{d(\xtwo,\mu_i)} + 1.$$
\end{lemma}
\begin{proof}
This essentially follows from application of Chernoff-Hoeffding bounds for concentration of $\hat{\mu}_i(t)$. Refer to Appendix \ref{app:E(t)} for details.
\end{proof}
\begin{lemma}
\label{lem:E1(t)}
$$\sum_{t=1}^T \PrR{i(t)=i, \overline{\Etheta_i(t)} , \Emu_i(t)} \le L_i(T) + 1.$$
\end{lemma}
\begin{proof}
This essentially follows from the observation that the beta-distributed random variable $\theta_i(t)$ is well-concentrated around its mean when $k_i(t)$ is large, that is, larger than $L_i(T)$. Refer to Appendix \ref{app:E1(t)} for details.
\end{proof}
\begin{lemma}
\label{lem:pBound}
Let $\tau_j$ denote the time step at which $j^{th}$ trial of first arm happens, then
$$\Ex[\frac{1}{p_{i,\tau_j+1}}] \le \left\{ \begin{array}{lcl}
1+\frac{3}{\Delta'_i}, & j< \frac{8}{\Delta'_i}, \\
1+ \Theta(e^{-\Delta'^2_ij/2} + \frac{1}{(j+1)\Delta'^2_i} e^{-D_ij} + \frac{1}{e^{\Delta'^2_ij/4}-1}), & j \ge \frac{8}{\Delta'_i},
\end{array}\right.$$
where $\Delta'_i=\mu_1-\xone$, $D_i=\xone\log \frac{\xone}{\mu_1} + (1-\xone) \log\frac{1-\xone}{1-\mu_1}$.
\end{lemma}
\begin{proof}
The proof of this inequality involves some careful algebraic manipulations using tight estimates for partial Binomial sums provided by \cite{jerabek}. Refer to Appendix \ref{app:pBound} for details.
\end{proof}
\paragraph{Proof of Theorem \ref{th:problem-dependent} and Theorem \ref{th:problem-independent}}
Let $\tau_k$ denote the time step at which arm $1$ is played for the $k^{th}$ time for $k\geq 1$, and let $\tau_0=0$.
Using the above lemmas,
\begin{eqnarray}
\Ex[k_i(T)] & = & \sum_{t=1}^T \Pr(i(t)=i) \nonumber\\
& = & \sum_{t=1}^T \Pr(i(t)=i, \Emu_i(t), \Etheta_i(t)) + \sum_{t=1}^T \PrR{i(t)=i, \Emu_i(t), \overline{\Etheta_i(t)}} + \sum_{t=1}^T \PrR{i(t) = i, \overline{\Emu_i(t)}} \nonumber\\
& \le & \sum_{t=1}^T \ExR{\frac{(1-p_{i,t})}{p_{i,t}} I(i(t)=1, \Etheta_i(t), \Emu_i(t))} + L_i(T) + 1 + \frac{1}{d(\xtwo,\mu_i)}+1\nonumber\\
(*) & \le & \sum_{k=0}^{T-1}  \ExR{ \frac{(1-p_{i,\tau_{k}+1})}{p_{i,\tau_{k}+1}} \sum_{t=\tau_{k}+1}^{\tau_{k+1}} I(i(t)=1)} + L_i(T) + 1 + \frac{1}{d(\xtwo,\mu_i)}+1\nonumber\\
 & = & \sum_{k=0}^{T-1}  \ExR{ \frac{1}{p_{i,\tau_{k}+1}}-1} + L_i(T) + 1 + \frac{1}{d(\xtwo,\mu_i)}+1\nonumber\\
& \le & \frac{24}{\Delta'^2_i}+ \sum_{j=0}^{T-1} \Theta\left(e^{-\Delta'^2_ij/2} + \frac{1}{(j+1)\Delta'^2_i} e^{-D_ij} + \frac{1}{e^{\Delta'^2_ij/4}-1}\right)+ L_i(T) + 1 + \frac{1}{d(\xtwo,\mu_i)}+1.\nonumber\\
\end{eqnarray}
The inequality marked $(*)$ uses the observation that $p_{i,t}=\Pr(\theta_1(t) > \xone | {\cal F}_{t-1})$ changes only when the distribution of $\theta_1(t)$ changes, that is, only on the time step after each play of first arm. Thus, $p_{i,t}$ is same at all time steps $t\in \{\tau_k+1, \ldots, \tau_{k+1}\}$, for every $k$.\newline\\

For {\bf logarithmic  problem-dependent bound of Theorem \ref{th:problem-dependent}}, for some $0<\epsilon\le1$, we set $\xtwo \in (\mu_i, \mu_1)$ such that $d(\xtwo,\mu_1)=d(\mu_i,\mu_1)/(1+\epsilon)$,
and set $\xone \in (\xtwo, \mu_1)$ such that 
$d(\xtwo, \xone)= d(\xtwo,\mu_1)/(1+\epsilon) = d(\mu_i,\mu_1)/(1+\epsilon)^2$ \ (\footnote{This way of choosing thresholds, in order to obtain bounds in terms of KL-divergences $d(\mu_i,\mu_1)$ rather than $\Delta_i$s, is inspired by \cite{Garivier11, Maillard11, Kaufmann12}.}). This gives 
$$L_i(T) = \frac{\ln T}{d(\xtwo, \xone)} =(1+\epsilon)^2 \frac{\ln T}{d(\mu_i,\mu_1)}.$$
Also, by some simple algebraic manipulations of the equality $d(\xtwo,\mu_1)=d(\mu_i,\mu_1)/(1+\epsilon)$, we can obtain
$$\xtwo-\mu_i \ge \frac{\epsilon}{(1+\epsilon)} \cdot \frac{d(\mu_i, \mu_1)}{\ln\left(\frac{\mu_1(1-\mu_i)}{\mu_i(1-\mu_1)}\right)}, $$
giving $$\frac{1}{d(\xtwo,\mu_i)} \le \frac{2}{(\xtwo-\mu_i)^2} = O(\frac{1}{\epsilon^2}).$$
Here order notation is hiding functions of $\mu_i$s and $\Delta_i$s, since they are assumed to be constants.
$$\sum_{j=0}^{T-1} \Theta(e^{-\Delta'^2_ij/2} + \frac{1}{(j+1)\Delta'^2_i} e^{-D_ij} + \frac{1}{e^{\Delta'^2_ij/4}-1}) \le \Theta\left(\frac{1}{\Delta'^2_i} + \frac{1}{\Delta'^2_i D} + \frac{1}{\Delta'^4_i} + \frac{1}{\Delta'^2_i}\right) = \Theta(1). $$
Combining, we get
\begin{eqnarray*}
\Ex[{\cal R}(T)] = \sum_i \Delta_i\Ex[k_i(T)] 
& = & \sum_i (1+\epsilon)^2 \frac{\ln T}{d(\mu_i, \mu_1)} \Delta_i + O(\frac{N}{\epsilon^2}) \le \sum_i (1+\epsilon')\frac{\ln T}{d(\mu_i, \mu_1)} \Delta_i + O(\frac{N}{\epsilon'^2}),
\end{eqnarray*}
where $\epsilon'=3\epsilon$, and the order notation in above hides $\mu_i$s and $\Delta_i$s in addition to the absolute constants.\newline\\

For obtaining $O(\sqrt{NT \ln T})$ {\bf problem-independent bound of Theorem \ref{th:problem-independent}}, we pick $\xtwo=\mu_i+\frac{\Delta_i}{3}, \xone=\mu_1-\frac{\Delta_i}{3}$, so that $\Delta'^2=(\mu_1-\xone)^2=\frac{\Delta_i^2}{9}$, and using Pinsker's inequality,
$d(\xtwo,\mu_i) \ge \frac{1}{2}(\xtwo-\mu_i)^2 = \frac{\Delta_i^2}{18}$, 
$d(\xtwo,\xone) \ge \frac{1}{2}(\xone-\xtwo)^2 \ge \frac{\Delta_i^2}{18}$. Then,
$$L_i(T)=\frac{\ln T}{d(\xtwo,\xone)} \le \frac{18\ln T}{\Delta_i^2}.$$
$$\frac{1}{d(x_i,\mu_i)} \le \frac{18}{\Delta_i^2}.$$
\begin{eqnarray*}
\sum_{j=0}^{T-1} \Theta\left(e^{-\Delta'^2_ij/2} + \frac{1}{(j+1)\Delta'^2_i} e^{-D_ij} + \frac{1}{e^{\Delta'^2_ij/4}-1}\right) & \le & \sum_{j=0}^{T-1} \Theta\left(e^{-\Delta'^2_ij/2} + \frac{1}{(j+1)\Delta'^2_i} + \frac{4}{\Delta'^2_i j}\right)\\
& = &  \Theta\left(\frac{1}{\Delta'^2_i} + \frac{\ln T}{\Delta'^2_i}\right)\\
& = &  \Theta\left(\frac{\ln T}{\Delta^2_i}\right).
\end{eqnarray*}
This gives,
$$\Ex[k_i(T)]=O\left(\frac{\ln T}{\Delta_i^2}\right).$$
\begin{eqnarray*}
\Ex[R(T)] = \sum_i \Delta_i\Ex[k_i(T)] & = & O\left(\sum_{i\ne 1} \frac{\ln T}{\Delta_i}\right) \\
\end{eqnarray*}
We observe that in the worst case, for all suboptimal $i$, $\Delta_i \ge \sqrt{\frac{N\ln T}{T}}$. This is because the total regret on playing arms with $\Delta_i < \sqrt{\frac{N\ln T}{T}}$ instead of the optimal arm is at most $\sqrt{NT \ln T}$. Thus, all the arms with $\Delta_i < \sqrt{\frac{N\ln T}{T}}$ can be assumed to be optimal arms. Also, in \cite{AgrawalG12} we proved that multiple optimal arms can only help. 

Therefore, substituting $\Delta_i=\sqrt{\frac{N\ln T}{T}}$,
$$\Ex[{\cal R}(T)] = O(\sqrt{NT \ln T})$$
\qed

\paragraph{Acknowledgement.} We thank Emil Je\v r\'abek  for telling us about his estimates of partial binomial sums. 
We also thank MathOverflow for connecting us with Emil. 
\bibliography{bibliography_contextual}
\bibliographystyle{abbrv}
\appendix

\section{Some results used in the proofs}

\begin{fact}[Chernoff-Hoeffding bound]
\label{fact:CHbound1}
Let $X_1, \ldots, X_n$ be independent $0-1$ r.v.s with $E[X_i] = p_i$ (not necessarily equal). Let $X = \frac{1}{n}\sum_i=X_i$, $\mu = E[X] = \frac{1}{n}\sum_{i=1}^n p_i$.
Then, for any $0<\lambda<1-\mu$,
$$\Pr(X\ge\mu+\lambda) \le \exp\{-nd(\mu+\lambda,\mu)\}, $$
and, for any $0<\lambda<\mu$,
$$\Pr(X\le \mu-\lambda) \le \exp\{-nd(\mu-\lambda,\mu)\}, $$
where $d(a,b) = a\ln\frac{a}{b} +(1-a) \ln \frac{(1-a)}{(1-b)}$.
\end{fact}

\begin{fact}[Chernoff--Hoeffding bound]
\label{fact:CHbound2}
Let $X_1, . . . , X_n$ be random variables with common range $[0, 1]$ and such that $\ExR{X_t \condR X_1, . . . , X_{t-1}} = \mu$. Let $S_n = X_1 + \ldots + X_n$. Then for
all $a\ge 0$,
$$\Pr(S_n \ge n\mu + a) \le e^{-2a^2/n},$$
$$\Pr(S_n \le n\mu - a) \le e^{-2a^2/n}.$$
\end{fact}

\begin{fact}
\label{fact:beta-binomial}
$$ F^{beta}_{\alpha, \beta}(y) = 1-F^B_{\alpha+\beta-1,y}(\alpha-1),$$
for all positive integers $\alpha, \beta$.
\end{fact}

\section{Proof of Lemma \ref{lem:E(t)}}
\label{app:E(t)}
Let $\tau_k$ denote the time at which $k^{th}$ trial of arm $i$ happens. Let $\tau_0=0$; Then,
\begin{eqnarray*}
\sum_{t=1}^T \Pr(i(t) = i, \overline{\Emu_i(t)}) & \le & \ExR{\sum_{k=1}^T \sum_{t=\tau_{k}+1}^{\tau_{s+1}} I(i(t)=i)I(\overline{\Emu_i(t)}) }\\
& = & \ExR{\sum_{k=0}^{T-1} I(\overline{\Emu_i(\tau_k+1)}) \sum_{t=\tau_{k}+1}^{\tau_{k+1}} I(i(t)=i)}\\
& = & \ExR{\sum_{k=0}^{T-1} I(\overline{\Emu_i(\tau_k+1)}) }\\
& \le & \ExR{\sum_{k=0}^{T-1} I(\overline{\Emu_i(\tau_k+1)}) }\\
& \le & 1+ \ExR{\sum_{k=1}^{T-1} I(\overline{\Emu_i(\tau_k+1)}) }\\
& \le & 1+ \sum_{k=1}^{T-1} \exp(-kd(\xtwo,\mu_i))\\
& \le & 1+\frac{1}{d(\xtwo,\mu_i)}\\
\end{eqnarray*}
The second last inequality follows from the observation that the event $\overline{\Emu_i(t)}$ was defined as $\hat{\mu}_i(t) > \xtwo$, where $\hat{\mu}_i(t)$ is the average of the outcomes observed from the plays of arm $i$ until time $t-1$. Thus at time $\tau_k+1$, $\hat{\mu}_i(\tau_k+1)$ is simply the average of the outcomes observed from $k$ i.i.d. plays of arm $i$, each of which is a Bernoulli trial with mean $\mu_i$. Using Chernoff-Hoeffding bounds (Fact \ref{fact:CHbound1}), we obtain that $\Pr(\hat{\mu}_i(\tau_k+1)>\xtwo) \le e^{-kd(\xtwo,\mu_i)}$.
\qed

\section{Proof of Lemma \ref{lem:E1(t)}}
\label{app:E1(t)}
\begin{eqnarray*}
\PrR{i(t)=i, \overline{\Etheta_i(t)} \condR  \Emu_i(t), {\cal F}_{t-1}} & \le & \PrR{\theta_i(t) > \xone \condR \hat{\mu}_i(t) \le \xtwo, {\cal F}_{t-1}}\\
&  = & \PrR{Beta(\hat{\mu}_i(t)k_i(t)+1, (1-\hat{\mu}_i(t))k_i(t)+1)  > \xone \condR \hat{\mu}_i(t) \le \xtwo} \\
&  \le & \PrR{Beta(\xtwo k_i(t)+1, (1-\xtwo)k_i(t)+1) > \xone} \\
& = & F^{B}_{k_i(t)+1, \xone}(\xtwo k_i(t))    \;\;\;\;\; (Fact~\ref{fact:beta-binomial})\\
& \le & F^{B}_{k_i(t), \xone}(\xtwo k_i(t))\\
& \le & e^{-k_i(t)d(\xtwo,\xone)},
\end{eqnarray*}
where the last inequality follows from Chernoff-Hoeffding bounds (refer to Fact \ref{fact:CHbound1}).
Therefore, for $t$ such that $k_i(t) >L_i(T)$,
$$\PrR{i(t)=i, \overline{\Etheta_i(t)} \condR  \Emu_i(t), {\cal F}_{t-1}} \le \frac{1}{T}.$$
Let $\tau$ be the largest time step until $k_i(t) \le L_i(T)$, then,

\begin{eqnarray*}
\sum_{t=1}^T \PrR{i(t)=i, \overline{\Etheta_i(t)}, \Emu_i(t)} & \le & \sum_{t=1}^T \PrR{i(t)=i, \overline{\Etheta_i(t)} \condR  \Emu_i(t)} \\
& = & \ExR{\sum_{t=1}^T \PrR{i(t)=i, \overline{\Etheta_i(t)} \condR  \Emu_i(t), {\cal F}_{t-1}}} \\
& = & \ExR{\sum_{t=1}^{\tau} \PrR{i(t)=i, \overline{\Etheta_i(t)} \condR \Emu_i(t), {\cal F}_{t-1}}  + \sum_{t=\tau+1}^{T} \PrR{i(t)=i, \overline{\Etheta_i(t)} \condR \Emu_i(t), {\cal F}_{t-1}}}\\
& \le & \ExR{\sum_{t=1}^{\tau} \PrR{i(t)=i, \overline{\Etheta_i(t)} \condR \Emu_i(t), {\cal F}_{t-1}}  + \sum_{t=\tau+1}^{T} \frac{1}{T}}\\
& \le & \ExR{\sum_{t=1}^{\tau} \PrR{i(t)=i, \overline{\Etheta_i(t)} \condR \Emu_i(t), {\cal F}_{t-1}}}  +1\\
& = & \ExR{\sum_{t=1}^{\tau} I(i(t)=i)}  +1\\
& \le & L_i(T) + 1.
\end{eqnarray*}
\qed


\section{Proof of Lemma \ref{lem:pBound}}
\label{app:pBound}
Let $k_1(t)=j, S_1(t)=s$. Let $y=\xone$. Then, $p_{i,t} = \Pr(\theta_1(t) > y) = F^B_{j+1,y}(s)$. Let $\tau_{j}+1$ denote the time step after the $(j)^{th}$ play of arm $1$. Then, $k_1(\tau_j+1) = j$, and
$$\Ex[\frac{1}{p_{i,\tau_j+1}}] = \sum_{s=0}^j \frac{f_{j,\mu_1}(s)}{F_{j+1, y}(s)}.$$
Let $\Delta'=\mu_1-y$.

\paragraph{For $j < \frac{8}{\Delta'}$:}
Let $R=\frac{\mu_1(1-y)}{y(1-\mu_1)}$, $D=y\log \frac{y}{\mu_1} + (1-y) \log\frac{1-y}{1-\mu_1}$. 
\begin{eqnarray}
 \sum_{s=0}^j \frac{f_{j,\mu_1}(s)}{F_{j+1, y}(s)} & \le & \frac{1}{1-y} \sum_{s=\lceil yj \rceil}^{j} \frac{f_{j,\mu_1}(s)}{F_{j,y}(s)}\nonumber\\
 & \le & \frac{1}{1-y} \sum_{s=0}^{\lfloor yj \rfloor} \frac{f_{j,\mu_1}(s)}{f_{j,y}(s)} + \frac{1}{1-y} \sum_{s=\lceil yj \rceil}^{j} 2 f_{j,\mu_1}(s) \nonumber\\
 & = & \frac{1}{1-y} \sum_{s=0}^{\lfloor yj \rfloor} R^{s} \frac{(1-\mu_1)^{j}}{(1-y)^j} + \frac{1}{1-y} \sum_{s=\lceil yj \rceil}^{j} 2 f_{j,\mu_1}(s) \nonumber\\
& \le & \frac{1}{1-y}  R^{yj+1} \frac{(1-\mu_1)^{j}}{(1-y)^j} + \frac{2}{\Delta'} \nonumber\\
& = & \frac{\mu_1}{\Delta'}e^{-Dj} +\frac{2}{\Delta'}\nonumber\\
& \le & \frac{3}{\Delta'}.
\end{eqnarray}

\paragraph{For $j \ge \frac{8}{\Delta'}$:} 
We will divide the sum $Sum(0,j)=\sum_{s=0}^j \frac{f_{j,\mu_1}(s)}{F_{j+1, y}(s)}$ into four partial sums and prove that
$$\begin{array}{lcl}
Sum(0, \lfloor yj \rfloor-1) & \le & \Theta\left(e^{-Dj}\frac{1}{(j+1)}\frac{1}{\Delta'^2}\right) + \Theta(e^{-2\Delta'^2j}),\\
Sum(\lfloor yj \rfloor, \lfloor yj \rfloor) & \le & 3e^{-Dj},\\
Sum(\lceil yj \rceil, \lfloor \mu_1j-\frac{\Delta'}{2}j \rfloor) & \le & \Theta(e^{-\Delta'^2j/2}),\\
Sum(\lceil \mu_1j-\frac{\Delta'}{2}j \rceil, j) & \le & 1+\frac{1}{e^{\Delta'^2j/4}-1}.\\
\end{array}$$

Together, the above estimates will prove the required bound.

We use the following bounds on the cdf of Binomial distribution \cite[Prop. A.4]{jerabek}.\\
For $s\le y(j+1)-\sqrt{(j+1)y(1-y)}$,
$$ F_{j+1,y}(s) = \Theta\left(\frac{y(j+1-s)}{y(j+1)-s} {j+1 \choose s} y^s (1-y)^{j+1-s}\right).$$
For $s\ge y(j+1)-\sqrt{(j+1)y(1-y)}$,
$$  F_{j+1,y}(s) = \Theta(1).$$

\paragraph{Bounding $Sum(0, \lfloor yj \rfloor-1)$.} Using the bounds just given, for any $s$,
\begin{eqnarray*}
\frac{f_{j,\mu_1}(s)}{F_{j+1,y}(s)} & \le & \Theta\left(\frac{f_{j,\mu_1}(s)}{\frac{y(j+1-s)}{y(j+1)-s}{j+1 \choose s} y^s (1-y)^{j+1-s}}\right) + \Theta(1) f_{j,\mu_1}(s)\\
& = & \Theta\left(\left(1-\frac{s}{y(j+1)}\right) \cdot R^s \cdot \frac{(1-\mu_1)^{j}}{(1-y)^{j+1}}\right) + \Theta(1) f_{j,\mu_1}(s).
\end{eqnarray*}

This gives
\begin{equation} \label{eqn:sum0}
Sum(0,{\lfloor yj \rfloor}-1)  \le \Theta\left( \frac{(1-\mu_1)^{j}}{(1-y)^{j+1}}\sum_{s=0}^{\lfloor yj \rfloor-1} \left(1-\frac{s}{y(j+1)}\right) \cdot R^s  \right) + \Theta(1) \sum_{s=0}^{\lfloor yj \rfloor-1} f_{j,\mu_1}(s).
\end{equation}

We now bound the first expression on the RHS.
\begin{eqnarray*}
\frac{(1-\mu_1)^{j}}{(1-y)^{j+1}}\sum_{s=0}^{\lfloor yj \rfloor-1} \left(1-\frac{s}{y(j+1)}\right) \cdot R^s & = & \frac{(1-\mu_1)^{j}}{(1-y)^{j+1}} \left( \frac{R^{\lfloor yj \rfloor}-1}{R-1} - \frac{1}{y(j+1)} \left( \frac{(\lfloor yj \rfloor-1) R^{\lfloor yj \rfloor}}{R-1}- \frac{R^{\lfloor yj \rfloor}-R}{(R-1)^2}\right)\right)\\
& \le &  \frac{(1-\mu_1)^{j}}{(1-y)^{j+1}}\left(\frac{1}{y(j+1)} \frac{R^{\lfloor yj \rfloor}}{(R-1)^2} + \frac{(y(j+1)-\lfloor yj\rfloor+1)}{y(j+1)} \frac{R^{\lfloor yj \rfloor}}{(R-1)}\right)\\
& \le & \frac{(1-\mu_1)^{j}}{(1-y)^{j+1}} \frac{3}{y(j+1)} \frac{R^{\lfloor yj \rfloor+1}}{(R-1)^2}\\
& \le & e^{-Dj} \frac{3}{y(1-y)(j+1)} \frac{R}{(R-1)^2}\\
\end{eqnarray*}
The last inequality uses 
$$\frac{(1-\mu_1)^{j}}{(1-y)^{j}} R^{\lfloor yj \rfloor} \le \frac{(1-\mu_1)^{j}}{(1-y)^{j}} R^{yj} = e^{-Dj}.$$
Now, $R-1 = \frac{\mu_1(1-y)}{y(1-\mu_1)}-1=\frac{\mu_1-y}{y(1-\mu_1)}$. And, $\frac{R}{R-1} =  \frac{\mu_1(1-y)}{\mu_1-y}$.
Therefore,

$$\frac{1}{y(1-y)(j+1)} \frac{R}{(R-1)^2} = \frac{1}{y(1-y)(j+1)} \cdot \frac{\mu_1(1-y)}{\mu_1-y}  \cdot \frac{y(1-\mu_1)}{\mu_1-y} = \frac{1}{(j+1)}\frac{\mu_1(1-\mu_1)}{(\mu_1-y)^2}.$$
Substituting, we get
\begin{eqnarray*}
\frac{(1-\mu_1)^{j}}{(1-y)^{j+1}}\sum_{s=0}^{\lfloor yj \rfloor} \left(1-\frac{s}{y(j+1)}\right) \cdot R^s & \le & e^{-Dj}\frac{1}{(j+1)}\frac{\mu_1(1-\mu_1)}{(\mu_1-y)^2}.
\end{eqnarray*}

Substituting in \eqref{eqn:sum0}
\begin{equation}
Sum(0,\lfloor yj \rfloor-1)  \le  \Theta\left(e^{-Dj}\frac{1}{(j+1)}\frac{1}{\Delta'^2}\right) + \Theta(1) \sum_{s=0}^{\lfloor yj \rfloor-1} f_{j,\mu_1}(s) \le \Theta\left(e^{-Dj}\frac{1}{(j+1)}\frac{1}{\Delta'^2}\right) + \Theta(e^{-2(\mu_1-y)^2j}).
\end{equation}

\paragraph{Bounding $Sum(\lfloor yj \rfloor, \lfloor yj \rfloor)$.} We use $\frac{f_{j,\mu_1}(s)}{F_{j+1,y}(s)} \le \frac{f_{j,\mu_1}(s)}{f_{j+1,y}(s)} = \left(1-\frac{s}{j+1}\right) R^s \frac{(1-\mu_1)^{j}}{(1-y)^{j+1}}$, to get
\begin{eqnarray}
Sum(\lfloor yj \rfloor, \lfloor yj \rfloor) & = & \frac{f_{j,\mu_1}(\lfloor yj \rfloor)}{F_{j+1,y}(\lfloor yj \rfloor)}\nonumber\\
& \le & \left(1-\frac{yj-1}{j+1}\right) R^{yj} \frac{(1-\mu_1)^{j}}{(1-y)^{j+1}} \nonumber\\
& \le & \frac{(1-y+\frac{2}{j+1})}{1-y} R^{yj} \frac{(1-\mu_1)^{j}}{(1-y)^{j}}\nonumber\\
& \le & 3e^{-Dj}.
\end{eqnarray}
The last inequality uses $j\ge \frac{1}{\Delta'} \ge \frac{1}{1-y}$.

\paragraph{Bounding $Sum(\lceil yj \rceil, \lfloor \mu_1j-\frac{\Delta'}{2}j \rfloor)$.}
Now, if $j>\frac{1}{\Delta'}$,then $\sqrt{(j+1)y(1-y)} >\sqrt{y} > y$, so $y(j+1)-\sqrt{(j+1)y(1-y)} < yj \le \lceil yj \rceil$. Therefore, for $s\ge \lceil yj \rceil$,
$  F_{j+1,y}(s) = \Theta(1).$ Using this observation, we derive the following.

\begin{eqnarray}
Sum(\lceil yj \rceil, \lfloor \mu_1j-\frac{\Delta'}{2}j \rfloor) & = &   \sum_{s=\lceil yj \rceil}^{\lfloor \mu_1j-\frac{\Delta'}{2}j \rfloor} \frac{f_{j,\mu_1}(s)}{F_{j+1,y}(s)}\nonumber\\
& = & \Theta\left( \sum_{s=\lceil yj \rceil}^{\lfloor \mu_1j-\frac{\Delta'}{2}j \rfloor} f_{j,\mu_1}(s)\right)\nonumber\\
& \le & \Theta(e^{-2\left(\mu_1j - \lfloor \mu_1j-\frac{\Delta'}{2}j \rfloor\right)^2/j})\nonumber\\
& = & \Theta(e^{-\Delta'^2j/2}),
\end{eqnarray}
where the inequality follows using Chernoff-Hoeffding bounds (refer to Fact \ref{fact:CHbound2}).

\paragraph{Bounding $Sum(\lceil \mu_1j-\frac{\Delta'}{2}j \rceil, j)$.}
For $s\ge \lceil \mu_1j-\frac{\Delta'}{2}j \rceil = \lceil yj+\frac{\Delta'}{2}j \rceil$, again using Chernoff-Hoeffding bounds from Fact \ref{fact:CHbound2},
$$F_{j+1,y}(s) \ge 1-e^{-2(yj+\frac{\Delta'}{2}j -y(j+1))^2/(j+1)} \ge 1-e^{2\Delta'} e^{-\Delta'^2j/2} \ge 1-e^{\Delta'^2j/4} e^{-\Delta'^2j/2} = 1-e^{-\Delta'^2j/4}.$$ 
The last inequality uses $j\ge \frac{8}{\Delta'}$.
\begin{eqnarray}
Sum(\lceil \mu_1j-\frac{\Delta'}{2}j \rceil, j)& = & \sum_{s=\lceil \mu_1j-\frac{\Delta'}{2}j \rceil}^{j} \frac{f_{j,\mu_1}(s)}{F_{j+1,y}(s)} \nonumber\\
&\le & \frac{1}{1-e^{-\Delta'^2j/4}}\nonumber\\
& =& 1+\frac{1}{e^{\Delta'^2j/4}-1}. 
\end{eqnarray}

Combining, we get for $j\ge \frac{8}{\Delta'}$,
$$ \Ex[\frac{1}{p_{i,\tau_{j+1}}}] \le 1+ \Theta(e^{-\Delta'^2j/2} + \frac{1}{(j+1)\Delta'^2} e^{-Dj} + \frac{1}{e^{\Delta'^2j/4}-1})$$
\qed

\end{document}